\PassOptionsToPackage{capitalize}{cleveref}
\PassOptionsToPackage{colorlinks}{hyperref}
\PassOptionsToPackage{dvipsnames,table,svgnames,x11names,RGB}{xcolor}
\documentclass{article}
\usepackage{iclr2025_conference}
\usepackage{times}
\usepackage[utf8]{inputenc}
\usepackage[T1]{fontenc}
\usepackage[caption=false]{subfig}
\usepackage[most]{tcolorbox}
\usepackage[toc]{appendix}
\usepackage{amsmath}
\usepackage{amsfonts}
\usepackage{amssymb}
\usepackage{arydshln}
\usepackage{bm}
\usepackage{booktabs}
\usepackage{caption}
\usepackage{enumitem}
\usepackage{wrapfig}
\usepackage{array}
\usepackage{colortbl}
\usepackage{algorithm}
\usepackage{algorithmic}
\usepackage{float}
\usepackage{graphicx}
\usepackage{mathtools}
\usepackage{microtype}
\usepackage{multirow}
\usepackage{makecell}
\usepackage{tablefootnote}
\usepackage{nicefrac}
\usepackage{pifont}
\usepackage{titlesec}
\usepackage{url}
\definecolor{Urlcolor}{RGB}{251,111,146}
\definecolor{Linkcolor}{RGB}{193,18,31}
\definecolor{CiteColor}{RGB}{32,126,190}
\definecolor{colorOne}{RGB}{72, 35, 116}     
\definecolor{colorTwo}{RGB}{32, 144, 140}    
\definecolor{colorThree}{RGB}{189, 222, 38}  
\usepackage{xspace}
\usepackage{stackengine}

\usepackage{amsthm}
\newtheorem{proposition}{Proposition} 

\usepackage[pagebackref,breaklinks,colorlinks,urlcolor=Urlcolor,linkcolor=Linkcolor,citecolor=CiteColor]{hyperref}
\usepackage{cleveref}
\usepackage{url}

\makeatletter
\renewcommand{\paragraph}{%
  \@startsection{paragraph}{4}%
  {\z@}{-0.5em}{-0.5em}%
  {\normalfont\normalsize\bfseries}%
}
\makeatother

\def\eg{\emph{e.g.}\xspace} 
\def\ie{\emph{i.e.}\xspace}

\newcommand{\tablestyle}[2]{\setlength{\tabcolsep}{#1}\renewcommand{\arraystretch}{#2}\centering\footnotesize}

\iclrfinalcopy 

\usepackage{amsmath,amsfonts,bm}

\def\eqref#1{equation~\ref{#1}}

\def\1{\bm{1}}

\DeclareMathAlphabet{\mathsfit}{\encodingdefault}{\sfdefault}{m}{sl}
\SetMathAlphabet{\mathsfit}{bold}{\encodingdefault}{\sfdefault}{bx}{n}

\title{
On Vanishing Variance in Transformer Length Generalization
}

\author{Ruining Li\thanks{Equal contribution; each reserves the right to be listed first.} \\
University of Oxford\\
\texttt{ruining@robots.ox.ac.uk} \\
\And
Gabrijel Boduljak$^*$ \\
University of Oxford \\
\texttt{gabrijel@robots.ox.ac.uk} \\
\And
Jensen (Jinghao) Zhou \\
University of Oxford \\
\texttt{jinghao@robots.ox.ac.uk} \\
}

\newcommand\rurl[1]{%
  \href{https://#1}{\nolinkurl{#1}}%
}

\begin{document}
\maketitle
\begin{abstract}
It is a widely known issue that Transformers, when trained on shorter sequences, fail to generalize robustly to longer ones at test time.
This raises the question of whether Transformer models are real \emph{reasoning} engines, despite their impressive abilities in mathematical problem solving and code synthesis.
In this paper, we offer a \emph{vanishing variance} perspective on this issue. To the best of our knowledge, we are the first to demonstrate that even for today's frontier models, a longer sequence length results in a decrease in variance in the output of the multi-head attention modules.
On the $\operatorname{argmax}$ retrieval and dictionary lookup tasks, our experiments show that applying layer normalization after the attention outputs leads to significantly better length generalization.
Our analyses attribute this improvement to a reduction---though not a complete elimination---of the distribution shift caused by vanishing variance.
Project page: \rurl{ruiningli.com/vanishing-variance}.

\end{abstract}
\begin{figure}[hbt!]
    \centering
    \includegraphics[scale=0.3]{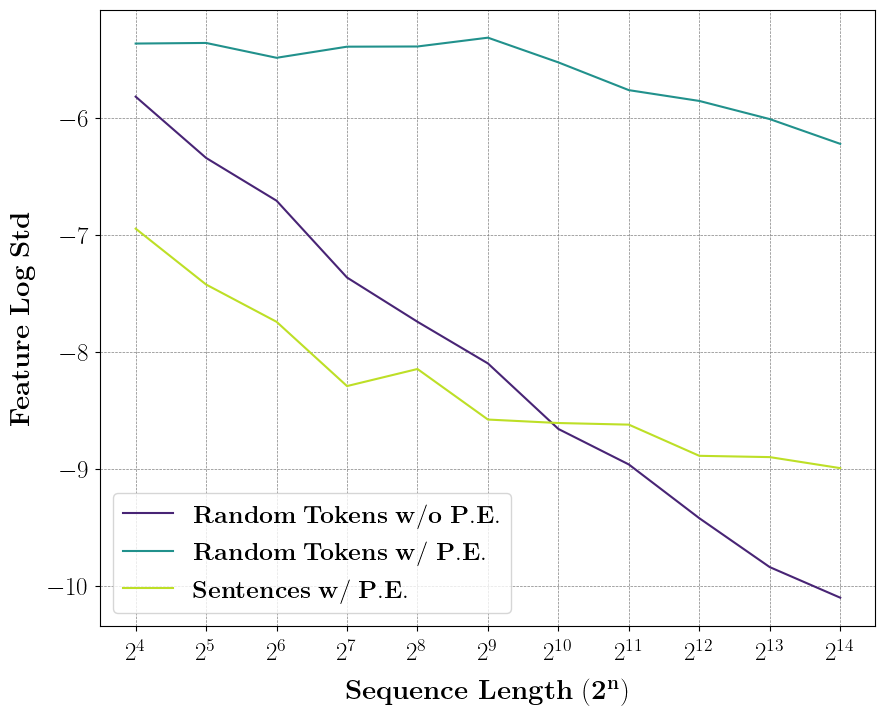}
    \caption{
    \textbf{Standard deviation of a fixed component in attention outputs from the first layer of Llama-3.2-1B (log-log scale) over multiple input sequences of fixed length $N$.}
    Even in the latest LLMs, increasing sequence length $N$ reduces the variance of attended outputs, significantly degrading accuracy on long sequences.
    }
    \label{fig:llama-std}
\end{figure}

\section{Background: Vanishing Variance}%
\label{sec:background}

It is no exaggeration to say that Transformers ~\citep{vaswani2017attention} is the most important architecture in modern deep learning.
It is widely adopted in almost every domain, ranging from natural language~\citep{vaswani2017attention, devlin2019bert, brown2020language} and vision~\citep{dosovitskiy2021an, peebles2023scalable} to audio~\citep{radford2023robust} and protein design~\citep{Jumper2021alphafold2}.
Despite its successes, recent studies~\citep{press2022train, zhou2023algorithms, zhou2024transformers, kazemnejad2024impact, velivckovic2024softmax} in large language models (LLMs) have shown that transformer-based models often struggle with length generalization, an ability that requires the model to generalize to longer sequences than seen during training.
Several prior works have proposed to either refine position encodings~\citep{ruoss2023randomized, zhou2024transformers, kazemnejad2024impact} or adapt the softmax function~\citep{press2022train, velivckovic2024softmax} to improve length generalization.
However, these methods are ad-hoc and lack interpretability, making it more of an art than a science to understand when and why they work.

In this paper, we study the distribution shift that occurs in the intermediate outputs when an attention module trained on shorter sequences is subsequently exposed to longer ones in a zero-shot manner.
We hope that our findings will encourage future research on network architectures that are provably robust (\eg, invariant) to varying sequence lengths.

\paragraph{Background and notations.}
At the core of Transformers is the attention mechanism ~\citep{vaswani2017attention}.
The attention first projects the input sequence
$\mathbf{X} = \left[\mathbf{x}_1 \Vert \mathbf{x}_2 \Vert \dots \Vert \mathbf{x}_N \right ]^{\top} \in \mathbb{R}^{N\times D}$, where $N$ is the sequence length and each item $\mathbf{x}_n \in \mathbb{R}^D$ has $D$ features, into keys $\mathbf{K} = \mathbf{X} \mathbf{W}_K \in \mathbb{R}^{N\times D}$ and values $\mathbf{V} = \mathbf{X} \mathbf{W}_V \in \mathbb{R}^{N\times D}$ using learnable weight matrices $\mathbf{W}_K, \mathbf{W}_V \in \mathbb{R}^{D\times D}$.
Similarly, the query sequence
$\mathbf{Y} = \left[\mathbf{y}_1 \Vert \mathbf{y}_2 \Vert \dots \Vert \mathbf{y}_M \right]^{\top} \in \mathbb{R}^{M\times D}$
is projected into queries $\mathbf{Q} = \mathbf{Y} \mathbf{W}_Q \in \mathbb{R}^{M\times D}$ using $\mathbf{W}_Q\in \mathbb{R}^{D\times D}$.
The attention then computes $\mathbf{O} = \operatorname{softmax} \left (\frac{\mathbf{Q} \mathbf{K}^{\top}}{\sqrt{D}} \right) \mathbf{V} \in \mathbb{R}^{M\times D}$ and projects it using another weight matrix $\mathbf{W}_O \in \mathbb{R}^{D\times D}$ to yield the final result $\operatorname{Attn}(\mathbf{X}, \mathbf{Y}) = \mathbf{O} \mathbf{W}_O$.
In this paper, we use the term ``attention weights'' to refer to the softmax score, \ie, $\operatorname{softmax}(\frac{\mathbf{Q} \mathbf{K}^{\top}}{\sqrt{D}})$, ``attention outputs'' the intermediate $\mathbf{O}$, and $\operatorname{Attn}(\mathbf{X}, \mathbf{Y})$ the final result.

Our main observation is the \emph{vanishing variance} problem: as the sequence length $N$ increases, the variance of attention outputs (computed over multiple input sequences of length $N$) decreases. We formalize this as~\cref{proposition}.
\begin{proposition}[The vanishing variance problem]
\label{proposition}
Consider a \textbf{trained} attention module with weights $\mathbf{W}_Q, \mathbf{W}_K, \mathbf{W}_V, \mathbf{W}_O$. Let $\mathbf{X} = \left[\mathbf{x}_1 \Vert \mathbf{x}_2 \Vert \dots \Vert \mathbf{x}_N \right ]^{\top}$ denote an input sequence of length $N$.
If (1) $\mathbf{x}_1, \mathbf{x}_2, \dots, \mathbf{x}_N \overset{\text{i.i.d}}{\sim} \mathcal{X}$, a distribution over a \textbf{finite} vocabulary, and (2) 
$\mathbb{E}_{\mathbf{x} \sim \mathcal{X}}[\mathbf{W}_V \mathbf{x}] = \mathbf{0}$, then for a \textbf{fixed} query $\mathbf{y}$ and a \textbf{fixed} feature $d$,
\begin{equation*}
    \begin{aligned}
\lim_{N \to \infty}
\operatorname{Var}_{ (\mathbf{x}_{1}, \mathbf{x}_{2}, \ldots, \mathbf{x}_{N})  \sim \mathcal{X}^N}
\left(\left[\operatorname{softmax}\left (\frac{\mathbf{Q}\mathbf{K}^\top}{\sqrt{D}} \right)\mathbf{V}\right]_{d} \right) = 0,
    \end{aligned}
\end{equation*}
where $\mathbf{x}_n, \mathbf{y} \in \mathbb{R}^D$ and $\mathbf{Q} \in \mathbb{R}^{1\times D}, \mathbf{K} \in \mathbb{R}^{N\times D}, \mathbf{V} \in \mathbb{R}^{N\times D}$ are intermediate results in $\operatorname{Attn}(\mathbf{X}, [\mathbf{y}])$.

Informally, for a fixed component $d$ in the attention outputs, its variance over input sequences of length $N$, where each sequence consists of $N$ independently and identically distributed (i.i.d.) tokens, vanishes as $N \to \infty$.
\end{proposition}
\begin{proof}
Please refer to~\cref{sec:app_proof}.
\end{proof}
Note that assumptions of~\cref{proposition} are violated in practice. In particular, the independence assumption does not hold in LLMs because of (1) the introduction of positional encoding, and more significantly (2) the nature of language, where preceding words provide important context for those that follow.
In addition, $\mathbb{E}[\mathbf{W}_V \mathbf{x}_i] = \mathbf{0}$ is not strictly enforced.
Nevertheless, we find that even for today's frontier LLMs, the decay in attention output variance, as established in~\cref{proposition}, remains pronounced.
In~\cref{fig:llama-std}, we plot the standard deviation $\sigma$ of a fixed component of the attention outputs from the first layer of Llama-3.2-1B~\citep{llama3modelcard} as a function of input sequence length $N$.
$\sigma$ is computed over $100$ length-$N$ sequences sampled randomly with $3$ strategies:
({\color{colorOne} Random Tokens w/o P.E.}) We sample single tokens i.i.d uniformly at random from the tokenizer's vocabulary, and remove the positional encoding for inference;
({\color{colorTwo} Random Tokens w/ P.E.}) We still sample single tokens i.i.d uniformly at random, but keep the positional encoding at inference time;
({\color{colorThree} Sentences w/ P.E.}) We sample consecutive sentences from a long paragraph\footnote{Obtained from \href{https://github.com/dscape/spell/blob/master/test/resources/big.txt}{https://github.com/dscape/spell/blob/master/test/resources/big.txt}}, and truncate the token sequences to length $N$---such sequences lie \emph{within} the LLM's training distribution.
As can be seen in the log-log plot, for {\color{colorOne} Random Tokens w/o P.E.}, where the independence assumption \emph{does} hold, $\sigma$ scales with input sequence length $N$ roughly as $\sigma \propto N^{-0.5}$.
For {\color{colorTwo} Random Tokens w/ P.E.} and {\color{colorThree} Sentences w/ P.E.}, where such assumption is no longer valid, the downward trend is still obvious.
\section{Layer Normalization for Length Generalization}%
\label{sec:method}


As variance vanishes with longer sequence lengths in attention outputs, we are intrigued to investigate the causes of performance degradation observed in LLMs. To this end, we perform a toy study on the statistical behavior of attention output values.

For simplicity, we consider a one-layer Transformer with single-head attention, omitting residual connections and normalization, following~\citet{velivckovic2024softmax}. 
We adopt this architecture as our \emph{Baseline}.
The model receives a \emph{single} query token and an input sequence of varying length to perform simple algorithmic tasks.
To eliminate confounds from positional encodings, we focus on order-invariant tasks, where the output depends only on the multiset (not the order) of input tokens, including $\operatorname{argmax}$ retrieval and dictionary lookup.
Our goal is to evaluate models trained on shorter sequences using longer (\ie, out-of-distribution in length) sequences to study length generalization.
More details of the model architecture and synthetic data generation are provided in~\cref{sec:supp_details}.

In~\cref{fig:hists}, we visualize the distribution of $5$ individual components in attention outputs $\mathbf{O}$ across multiple input sequences of lengths $2^4$, $2^{12}$ and $2^{14}$, obtained with a model checkpoint trained on sequences of up to length $2^4$.
As can be seen in the top row, testing on out-of-distribution sequence lengths leads to vanishing variance, causing a distribution shift where each individual component of $\mathbf{O}$ becomes more concentrated around its mean.

\begin{figure}[htbp]
    \centering
    \includegraphics[width=\textwidth]{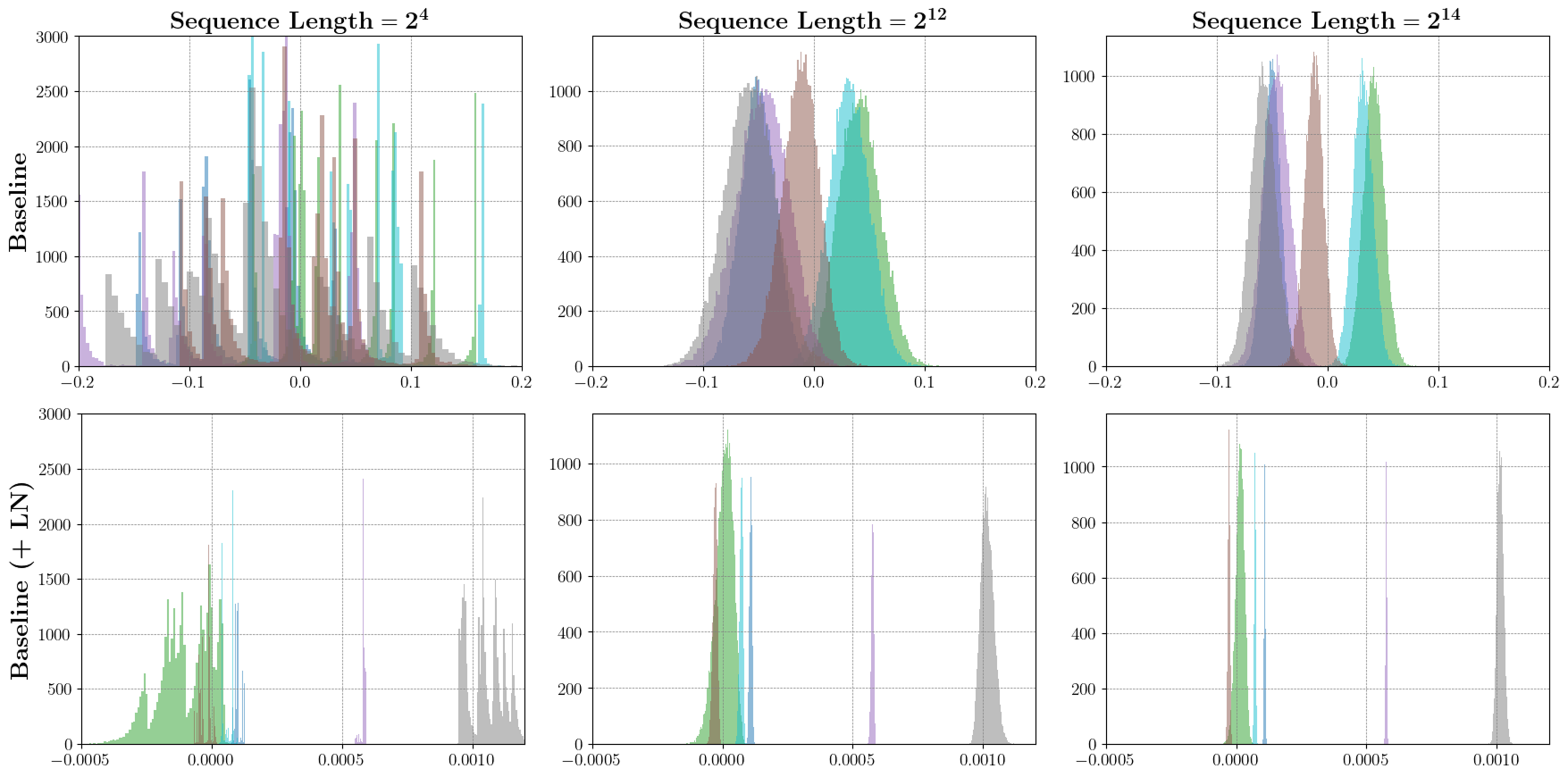}

    \caption{
    \textbf{Distribution of 5 individual features in attention outputs $\mathbf{O}$ across batches.}
    Each color represents a different feature.
    As input sequence length $N$ increases from $2^4$ to $2^{14}$, feature variance decreases, and values concentrate around their mean.
    Layer normalization (bottom) scales and shifts features to maintain relatively constant \emph{global} variance, likely explaining its superior length generalization compared to the \emph{Baseline} (top).
    }
    \label{fig:hists}
\end{figure}

While this distribution shift of individual features is expected according to~\cref{proposition}, we are more interested in the distribution shift of the entire feature vector in $\mathbb{R}^D$, as the whole vector is subsequently input to an MLP to predict the final result. As input sequence length $N$ increases, each feature is less likely to have extreme values (as its distribution is more centered). Consequently, the \emph{global} feature variance, defined as $\sigma_{\text{global}}^2 = \frac{1}{D}\Sigma_{d=1}^D (\mathbf{o}_d - \mu_{\text{global}})^2$ where $\mu_{\text{global}}=\frac{1}{D}\Sigma_{d=1}^D \mathbf{o}_d$ is the global mean, also decreases.
We illustrate this observation in~\cref{fig:distshift} (right), where the global variance decays as $N$ increases.
In~\cref{fig:distshift} (left), we show that in addition to the global variance, the global mean $\mu_{\text{global}}$ also exhibits drift.
Such a distribution shift in attention outputs (and thus MLP inputs) hinders generalization, since the MLP is only trained on features with larger global variance and a different global mean~\citep{zhou2022domain}.

\begin{figure}
    \centering
    \includegraphics[width=0.95\linewidth]{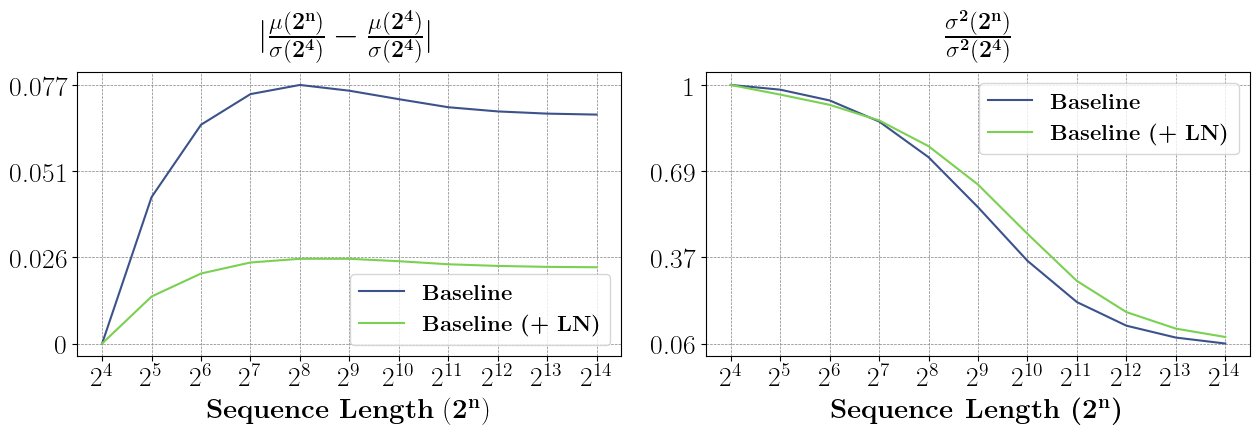}
    \caption{
    \textbf{Layer normalization helps mitigate distribution shift in attention outputs.}
    (\textbf{Left}) shows the drift in global mean as input sequence length deviates from the training distribution. The mean is normalized by the training global variance to eliminate scale differences.
    (\textbf{Right}) shows the decay in global variance.
    All results are averaged across $32$k random input sequences of the fixed length.}
    \label{fig:distshift}
\end{figure}

To mitigate this distribution shift, we explore applying layer normalization~\citep{ba2016layer} immediately after the attention outputs,
\ie, $\operatorname{LayerNorm}(\mathbf{O})_{b, t, d} = \gamma_d \cdot \frac{\mathbf{o}_{b, t, d} - \mathbf{\mu}_{b, t}}{\mathbf{\sigma}_{b, t} + \epsilon} + \beta_d$, where $\mathbf{O}$ is the batched attention outputs, $\mu_{b,t} = \frac{1}{D}\Sigma_{d=1}^D \mathbf{O}_{b,t,d}$, $\sigma_{b,t} = \sqrt{\frac{1}{D}\Sigma_{d=1}^D (\mathbf{O}_{b,t,d} - \mu_{b,t})^2}$, and $\gamma, \beta \in \mathbb{R}^D$ are learnable scale and shift parameters. 
While variance decay in individual features is inevitable (bottom row of \cref{fig:hists}), standardization and learnable scale and shift parameters help stabilize the feature distribution. This adjustment preserves the global mean and variance more effectively as sequence length increases $1000\times$ (\cref{fig:distshift}).
This enhances length generalization, as discussed next.

\section{Experiments}%
\label{sec:experiments}

We consider \emph{two} tasks: $\operatorname{argmax}$ retrieval and dictionary lookup. The former has been considered by~\citet{velivckovic2024softmax}. The latter closely resembles the core function of the attention mechanism (\ie, to retrieve the most relevant information based on the similarity between queries and keys).
As detailed in~\cref{sec:supp_details}, the order of input tokens does not affect the target output in either task.
By deliberately selecting such tasks, we isolate and examine the length generalization capabilities of the attention mechanism itself, independent of any effects introduced by positional encodings~\citep{zhou2024transformers}.
We generate synthetic data (of input sequence length up to $16$) to train the models, and evaluate them on sequences of length up to $2^{14}$.

\subsection{Results and Analysis}
The results presented in~\cref{tab:comparison-argmax} and~\cref{tab:comparison-dict-lookup} indicate that applying layer normalization to attention outputs leads to consistently better accuracy on out-of-distribution sequence lengths, with statistical significance confirmed by a paired $t$-test over $100$ training runs from different random seeds.

Test-time adaptation and fine-tuning are common techniques for improving length generalization in transformers~\citep{anil2022exploring, velivckovic2024softmax}.
To show that the benefits of layer normalization are \emph{orthogonal} to these techniques, we implement the adaptive temperature method from~\citet{velivckovic2024softmax} in both architectures, with and without layer normalization. Combined with test-time adaptation, layer normalization still yields a significant improvement. In~\cref{fig:heatmap}, we demonstrate that layer normalization also mitigates dispersion~\citep{velivckovic2024softmax}.


\begin{figure}[htb!]
    \centering
    \includegraphics[width=\linewidth]{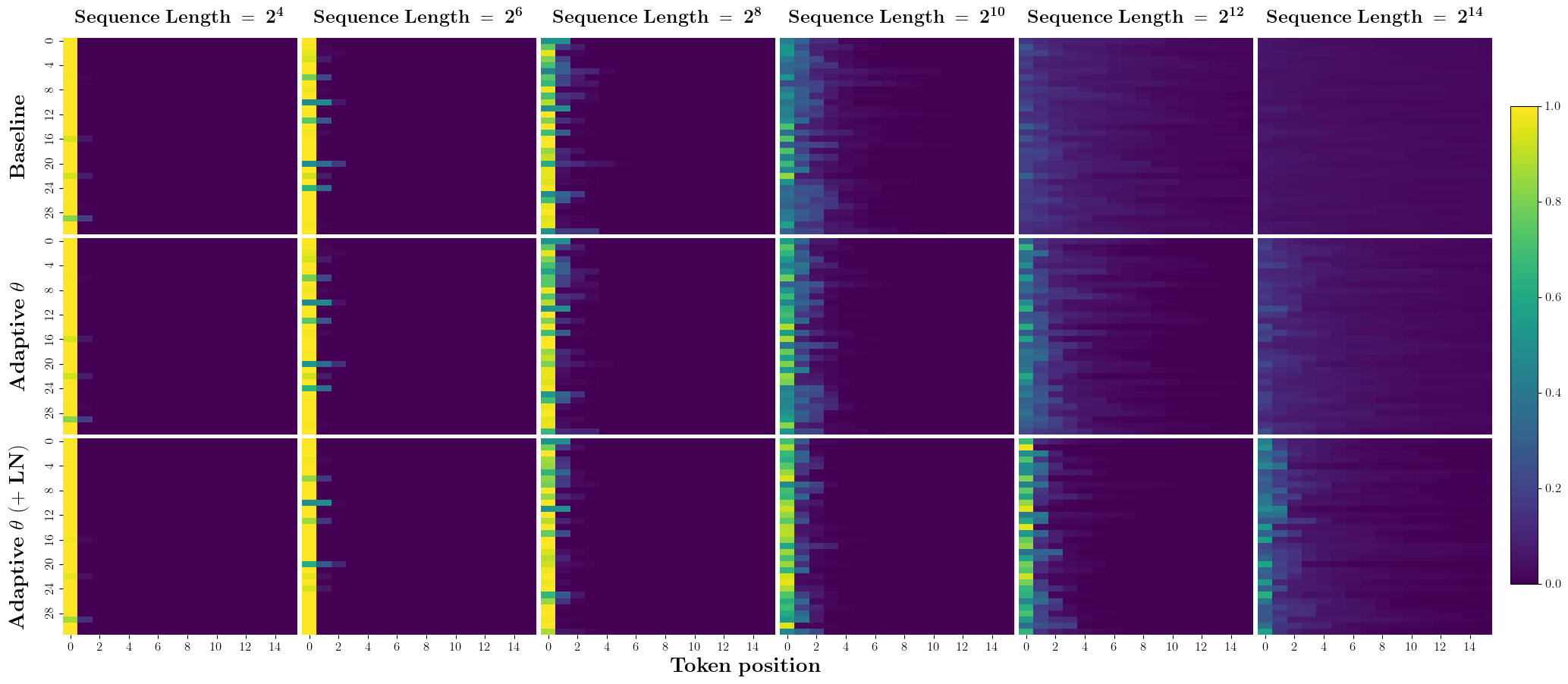}
    \caption{\textbf{Heatmap} of the largest $16$ attention weights, computed over $32$ examples.
    Layer normalization mitigates dispersion, which is inevitable as sequence length increases~\citep{velivckovic2024softmax}.
    }
    \label{fig:heatmap}
\end{figure}

\begin{table}[]
    \caption{
    \textbf{Results (\%) on the $\operatorname{argmax}$ retrieval task.} 
    Results are averaged over $100$ runs with different random seeds. 
    $p$-values are computed using a paired $t$-test.
    Entries highlighted in \colorbox{green!15}{green} indicate those with in-distribution sequence lengths.
    }
    \tablestyle{3.0pt}{1.0}
    \centering
    \begin{tabular}{lccccccccccc}
        \toprule
        {\bf Model} & \cellcolor{green!15} $2^4$ & $2^5$ & $2^6$ & $2^7$ & $2^8$ & $2^9$ & $2^{10}$ & $2^{11}$ & $2^{12}$ & $2^{13}$ & $2^{14}$\\
        \midrule
        \multicolumn{7}{l}{\textit{w.o. test-time adaptation}} \\
        Baseline & \cellcolor{green!15} {$\mathbf{99.6} $} & $99.2 $ & ${98.4} $ & $96.8 $ & $93.7 $ & $88.0 $ & $78.0 $ & $62.5 $ & $44.2 $ & $29.7 $ & $20.8 $ \\
        Baseline (+ LN) & \cellcolor{green!15} $\mathbf{99.6} $ & $\mathbf{99.3} $ & $\mathbf{98.6} $ & $\mathbf{97.4} $ & $\mathbf{94.8} $ & $\mathbf{89.8} $ & $\mathbf{81.0} $ & $\mathbf{66.9} $ & $\mathbf{49.2} $ & $\mathbf{33.0} $ & $\mathbf{22.6} $ \\
        \textcolor{gray}{$p$-value} & \cellcolor{green!15} \textcolor{gray}{$8/10^{1}$} & \textcolor{gray}{$4/10^{2}$} & \textcolor{gray}{$4/10^{2}$} & \textcolor{gray}{$2/10^{2}$} & \textcolor{gray}{$4/ 10^{5}$} & \textcolor{gray}{$2/10^{4}$} & \textcolor{gray}{$1/10^{4}$} & \textcolor{gray}{$2/10^5$} & \textcolor{gray}{$2/10^6$} & \textcolor{gray}{$4/10^5$} & \textcolor{gray}{$3/10^3$} \\
        \midrule
        \multicolumn{7}{l}{\textit{w. test-time adaptation}} \\
        Adaptive $\theta$ & \cellcolor{green!15} {$99.6 $} & $99.2 $ & ${98.5} $ & $96.9 $ & $94.1 $ & $89.1 $ & $81.2 $ & $69.1 $ & $54.2 $ & $39.0 $ & $27.1 $ \\
        Adaptive $\theta$ (+ LN) & \cellcolor{green!15} $\mathbf{99.7} $ & $\mathbf{99.4} $ & $\mathbf{98.7} $ & $\mathbf{97.5} $ & $\mathbf{95.1} $ & $\mathbf{91.0} $ & $\mathbf{84.0} $ & $\mathbf{73.6} $ & $\mathbf{58.9} $ & $\mathbf{43.1} $ & $\mathbf{30.4} $ \\
        \textcolor{gray}{$p$-value} & \cellcolor{green!15} \textcolor{gray}{$7/10^1$} & \textcolor{gray}{$5/10^{4}$} & \textcolor{gray}{$1/10^2$} & \textcolor{gray}{$5/10^5$} & \textcolor{gray}{$4/10^4$} & \textcolor{gray}{$5/10^{5}$} & \textcolor{gray}{$1/10^{4}$} & \textcolor{gray}{$2/10^5$} & \textcolor{gray}{$8/10^5$} & \textcolor{gray}{$4/10^4$} & \textcolor{gray}{$7/10^4$} \\
        \bottomrule 
    \end{tabular}
    \label{tab:comparison-argmax}
\end{table}
\begin{table}[]
    \caption{
    \textbf{Results (\%) on the dictionary lookup task}.
    Results are averaged over $100$ runs with different random seeds. 
    $p$-values are computed using a paired $t$-test. 
    Entries highlighted in \colorbox{green!15}{green} indicate those with in-distribution sequence lengths.
    }
    \tablestyle{2.4pt}{1.0}
    \centering
    \begin{tabular}{lccccccccccc}
        \toprule
        {\bf Model} & \cellcolor{green!15} $2^4$ & $2^5$ & $2^6$ & $2^7$ & $2^8$ & $2^9$ & $2^{10}$ & $2^{11}$ & $2^{12}$ & $2^{13}$ & $2^{14}$\\
        \midrule
        \multicolumn{7}{l}{\textit{w.o. test-time adaptation}} \\
        Baseline & \cellcolor{green!15} {$99.3 $} & $98.6 $ & ${97.3} $ & $94.7 $ & $89.5 $ & $80.4$ & $67.6 $ & $52.9 $ & $38.7 $ & $26.5 $ & $17.8 $ \\
        Baseline (+ LN) & \cellcolor{green!15} $\mathbf{99.4} $ & $\mathbf{98.8} $ & $\mathbf{97.6} $ & $\mathbf{95.3} $ & $\mathbf{90.7} $ & $\mathbf{82.9} $ & $\mathbf{71.7} $ & $\mathbf{57.7} $ & $\mathbf{44.1} $ & $\mathbf{32.3} $ & $\mathbf{22.4} $ \\
        \textcolor{gray}{$p$-value} & \cellcolor{green!15} \textcolor{gray}{$6/10^2$}& \textcolor{gray}{$1/10^2$} & \textcolor{gray}{$5/10^2$} & \textcolor{gray}{$2/10^3$} & \textcolor{gray}{$2/10^{4}$} & \textcolor{gray}{$3/10^{8}$} & \textcolor{gray}{$1/10^{11}$} & \textcolor{gray}{$2/10^{12}$} & \textcolor{gray}{$7/10^{15}$} & \textcolor{gray}{$3/10^{21}$} & \textcolor{gray}{$2/10^{19}$} \\
        \midrule
        \multicolumn{7}{l}{\textit{w. test-time adaptation}} \\
        Adaptive $\theta$ & \cellcolor{green!15} {$99.3 $} & $98.6 $ & ${97.2} $ & $94.5 $ & $89.3 $ & $80.4 $ & $67.8 $ & $52.6 $ & $38.6 $ & $27.3 $ & $20.8 $ \\
        Adaptive $\theta$ (+ LN) & \cellcolor{green!15} $\mathbf{99.4} $ & $\mathbf{98.8} $ & $\mathbf{97.6} $ & $\mathbf{95.4} $ & $\mathbf{90.6} $ & $\mathbf{82.9} $ & $\mathbf{71.7} $ & $\mathbf{57.8} $ & $\mathbf{44.5} $ & $\mathbf{33.4} $ & $\mathbf{27.7} $ \\
        \textcolor{gray}{$p$-value} & \cellcolor{green!15} \textcolor{gray}{$6/10^1$} & \textcolor{gray}{$5/10^2$} & \textcolor{gray}{$3/10^2$} & \textcolor{gray}{$1/10^4$} & \textcolor{gray}{$3/10^{4}$} & \textcolor{gray}{$1/10^{5}$} & \textcolor{gray}{$8/10^{9}$} & \textcolor{gray}{$6/10^{12}$} & \textcolor{gray}{$2/10^{16}$} & \textcolor{gray}{$9/10^{20}$} & \textcolor{gray}{$1/10^{21}$} \\
        \bottomrule 
    \end{tabular}
    \label{tab:comparison-dict-lookup}
\end{table}


\paragraph{Does layer normalization alleviate distribution shift?}
Layer normalization does alleviate---but \emph{not} eliminate---distribution shift. 
With layer normalization, the global mean and global variance remain more stable on out-of-distribution sequence lengths (\cref{fig:distshift}).
However, the variance of fixed components in attention outputs still decays, regardless of layer normalization (\cref{fig:hists}).

\subsection{Ablations}

\label{sec:ablations}
In addition to layer normalization, we explore an alternative normalization strategy in which we \emph{standardize} (\ie, std. in \cref{tab:ablation-argmax}) the attention outputs across the $D$ features without the learnable scale and shift parameters present in LN, \ie, $
\operatorname{Standardize}(\mathbf{O})_{b, t, d} = \frac{\mathbf{O}_{b, t, d} - \mathbf{\mu}_{b, t}}{\mathbf{\sigma}_{b, t} + \epsilon},
$
where $\mu_{b,t}$ and $\sigma_{b,t}$ are computed in the same manner as in layer normalization.

As shown in~\cref{tab:ablation-argmax}, where the relative accuracy gain over \emph{Baseline} on the $\arg\max$ retrieval task is reported, standardization improves length generalization, even though it strictly constrains model capacity.
This underscores the importance (and potential benefits) of addressing the observed distribution shift.
LN outperforms standardization, as confirmed by the paired $t$-test.
Similar ablation results on the dictionary lookup task can be found in~\cref{sec:additional-ablations}.

\begin{table}[htb]
    \caption{
    \textbf{Ablations} on different normalization strategies on the $\operatorname{argmax}$ retrieval task.
    Relative results (\%) compared to the \emph{Baseline} ($\triangle$) are reported.
    }
    \tablestyle{3.0pt}{1.0}
    \centering
    \begin{tabular}{lccccccccccc}
        \toprule
        {\bf Model} & \cellcolor{green!15} $2^4$ & $2^5$ & $2^6$ & $2^7$ & $2^8$ & $2^9$ & $2^{10}$ & $2^{11}$ & $2^{12}$ & $2^{13}$ & $2^{14}$\\
        \midrule
        $\triangle$ (+ std.) & \cellcolor{green!15} $-0.05 $ & $+0.00 $ & $+0.09 $ & $+0.26$ & $+0.60 $ & $+0.84 $ & $+1.62 $ & $+2.26 $ & $+3.05 $ & $+1.80 $ & $+0.70 $ \\
        $\triangle$ (+ LN) & \cellcolor{green!15} $\mathbf{+0.01} $ & $\mathbf{+0.11} $ & $\mathbf{+0.21} $ & $\mathbf{+0.57} $ & $\mathbf{+1.15} $ & $\mathbf{+1.81} $ & $\mathbf{+2.98} $ & $\mathbf{+4.32} $ & $\mathbf{+4.99} $ & $\mathbf{+3.30} $ & $\mathbf{+1.76} $ \\    
        \textcolor{gray}{$p$-value} & \cellcolor{green!15} \textcolor{gray}{$2/10^3$} & \textcolor{gray}{$7/10^4$} & \textcolor{gray}{$1/10^2$} & \textcolor{gray}{$5/10^4$} & \textcolor{gray}{$3/10^4$} & \textcolor{gray}{$3/10^4$} & \textcolor{gray}{$7/10^4$} & \textcolor{gray}{$2/10^4$} & \textcolor{gray}{$3/10^4$} & \textcolor{gray}{$1/10^3$} & \textcolor{gray}{$5/10^3$} \\
        \bottomrule 
    \end{tabular}
    \label{tab:ablation-argmax}
\end{table}



\section{Related Work}

\paragraph{Positional encoding for length generalization.}
Many works have attributed the inability of Transformers to extrapolate to longer sequences to positional encoding.
Several alternatives to the sinusoidal positional encoding originally introduced by \citet{vaswani2017attention} have been proposed to enhance the performance of Transformer-based models in natural language processing (NLP) tasks, including relative positional encoding~\citep{shaw2018self, dai2019transformer}, rotary positional encoding~\citep{su2024roformer}, no positional encoding~\citep{haviv2022transformer} and randomized positional encoding~\citep{ruoss2023randomized}.
Authors have examined the impact of different variants of positional encoding on length generalization~\citep{chi2022kerple, ruoss2023randomized, kazemnejad2024impact, li2024functional, peng2024yarn, zhou2024transformers}.
Unlike prior work that explores positional encoding for length generalization, we focus on algorithmic tasks that are order-invariant.
We present a \emph{vanishing variance} perspective on length generalization which is orthogonal to the extrapolability of positional encoding.

\paragraph{Alternatives to softmax attention.}
The $\operatorname{softmax}$ output has been utilized to interpret the inner workings of Transformers~\citep{xu2015show, choi2016retain, martins2016softmax}.
More recently, \citet{velivckovic2024softmax} demonstrated that the attention weights output by $\operatorname{softmax}$ will \emph{disperse} as sequence length increases, attributing this phenomenon to the Transformer's limited capability in length generalization.
In this paper, we show that this dispersion leads to the vanishing variance problem in the intermediate attention outputs.
While many variants of softmax attention have been introduced~\citep{correia2019adaptively, press2022train, tan2024stick, ye2024differential}, they are motivated mostly by interpretability, rather than the distribution of attention outputs for length generalization.
To the best of our knowledge, none of the existing works have fundamentally eliminated the vanishing variance problem we presented in this paper.
We hope our study can motivate designs of network architectures that are provably invariant to sequence length variations.

\section{Conclusion}%
\label{sec:conclusion}

In this paper, we have introduced the \emph{vanishing variance} problem and provided both theoretical analysis and empirical evidence demonstrating its role in inducing distribution shift in attention outputs. This shift hinders the ability of Transformers to generalize effectively to out-of-distribution sequence lengths.
We demonstrated that mitigating this distribution shift through techniques like layer normalization and standardization---despite potential trade-offs in model expressiveness---significantly improves length generalization in attention models.


\paragraph{Future work.}
We conduct our experiments using a single-layer, single-head attention architecture for simplicity, while real-world models typically use multi-layer, multi-head attention. Our conclusions may not fully generalize to these more complex architectures. Future work may validate the normalization strategies on larger benchmarks like CLRS~\citep{velivckovic2022clrs} and real-world LLMs. 
Moreover, layer normalization only \emph{partially} mitigates distribution shift presented in this paper, and is already widely adopted in Transformers (though not immediately after attention outputs).
Future work may design architectures that are provably invariant to sequence length variations.

\bibliography{main}

\begin{thebibliography}{31}
\providecommand{\natexlab}[1]{#1}
\providecommand{\url}[1]{\texttt{#1}}
\expandafter\ifx\csname urlstyle\endcsname\relax
  \providecommand{\doi}[1]{doi: #1}\else
  \providecommand{\doi}{doi: \begingroup \urlstyle{rm}\Url}\fi

\bibitem[AI@Meta(2024)]{llama3modelcard}
AI@Meta.
\newblock Llama 3 model card.
\newblock 2024.
\newblock URL \url{https://github.com/meta-llama/llama3/blob/main/MODEL_CARD.md}.

\bibitem[Anil et~al.(2022)Anil, Wu, Andreassen, Lewkowycz, Misra, Ramasesh, Slone, Gur-Ari, Dyer, and Neyshabur]{anil2022exploring}
Cem Anil, Yuhuai Wu, Anders Andreassen, Aitor Lewkowycz, Vedant Misra, Vinay Ramasesh, Ambrose Slone, Guy Gur-Ari, Ethan Dyer, and Behnam Neyshabur.
\newblock Exploring length generalization in large language models.
\newblock \emph{Advances in Neural Information Processing Systems}, 35:\penalty0 38546--38556, 2022.

\bibitem[Ba et~al.(2016)Ba, Kiros, and Hinton]{ba2016layer}
Jimmy~Lei Ba, Jamie~Ryan Kiros, and Geoffrey~E. Hinton.
\newblock Layer normalization.
\newblock \emph{arXiv preprint arXiv:1607.06450}, 2016.

\bibitem[Brown et~al.(2020)Brown, Mann, Ryder, Subbiah, Kaplan, Dhariwal, Neelakantan, Shyam, Sastry, Askell, et~al.]{brown2020language}
Tom Brown, Benjamin Mann, Nick Ryder, Melanie Subbiah, Jared~D Kaplan, Prafulla Dhariwal, Arvind Neelakantan, Pranav Shyam, Girish Sastry, Amanda Askell, et~al.
\newblock Language models are few-shot learners.
\newblock In \emph{NeurIPS}, 2020.

\bibitem[Chi et~al.(2022)Chi, Fan, Ramadge, and Rudnicky]{chi2022kerple}
Ta-Chung Chi, Ting-Han Fan, Peter~J Ramadge, and Alexander Rudnicky.
\newblock Kerple: Kernelized relative positional embedding for length extrapolation.
\newblock \emph{NeurIPS}, 2022.

\bibitem[Choi et~al.(2016)Choi, Bahadori, Sun, Kulas, Schuetz, and Stewart]{choi2016retain}
Edward Choi, Mohammad~Taha Bahadori, Jimeng Sun, Joshua Kulas, Andy Schuetz, and Walter Stewart.
\newblock Retain: An interpretable predictive model for healthcare using reverse time attention mechanism.
\newblock \emph{NeurIPS}, 2016.

\bibitem[Correia et~al.(2019)Correia, Niculae, and Martins]{correia2019adaptively}
Gon{\c{c}}alo~M Correia, Vlad Niculae, and Andr{\'e}~FT Martins.
\newblock Adaptively sparse transformers.
\newblock \emph{arXiv preprint arXiv:1909.00015}, 2019.

\bibitem[Dai et~al.(2019)Dai, Yang, Yang, Carbonell, Le, and Salakhutdinov]{dai2019transformer}
Zihang Dai, Zhilin Yang, Yiming Yang, Jaime~G Carbonell, Quoc Le, and Ruslan Salakhutdinov.
\newblock Transformer-xl: Attentive language models beyond a fixed-length context.
\newblock In \emph{ACL}, 2019.

\bibitem[Devlin et~al.(2019)Devlin, Chang, Lee, and Toutanova]{devlin2019bert}
Jacob Devlin, Ming-Wei Chang, Kenton Lee, and Kristina Toutanova.
\newblock {BERT}: Pre-training of deep bidirectional transformers for language understanding.
\newblock In \emph{NAACL}, 2019.

\bibitem[Dosovitskiy et~al.(2021)Dosovitskiy, Beyer, Kolesnikov, Weissenborn, Zhai, Unterthiner, Dehghani, Minderer, Heigold, Gelly, Uszkoreit, and Houlsby]{dosovitskiy2021an}
Alexey Dosovitskiy, Lucas Beyer, Alexander Kolesnikov, Dirk Weissenborn, Xiaohua Zhai, Thomas Unterthiner, Mostafa Dehghani, Matthias Minderer, Georg Heigold, Sylvain Gelly, Jakob Uszkoreit, and Neil Houlsby.
\newblock An image is worth 16x16 words: Transformers for image recognition at scale.
\newblock In \emph{ICLR}, 2021.

\bibitem[Haviv et~al.(2022)Haviv, Ram, Press, Izsak, and Levy]{haviv2022transformer}
Adi Haviv, Ori Ram, Ofir Press, Peter Izsak, and Omer Levy.
\newblock Transformer language models without positional encodings still learn positional information.
\newblock In \emph{EMNLP}, 2022.

\bibitem[Jumper et~al.(2021)Jumper, Evans, Pritzel, Green, Figurnov, Ronneberger, Tunyasuvunakool, Bates, Ž{\'i}dek, Potapenko, Bridgland, Meyer, Kohl, Ballard, Cowie, Romera-Paredes, Nikolov, Jain, Adler, Back, Petersen, Reiman, Clancy, Zielinski, Steinegger, Pacholska, Berghammer, Bodenstein, Silver, Vinyals, Senior, Kavukcuoglu, Kohli, and Hassabis]{Jumper2021alphafold2}
John~M. Jumper, Richard Evans, Alexander Pritzel, Tim Green, Michael Figurnov, Olaf Ronneberger, Kathryn Tunyasuvunakool, Russ Bates, Augustin Ž{\'i}dek, Anna Potapenko, Alex Bridgland, Clemens Meyer, Simon A~A Kohl, Andy Ballard, Andrew Cowie, Bernardino Romera-Paredes, Stanislav Nikolov, Rishub Jain, Jonas Adler, Trevor Back, Stig Petersen, David Reiman, Ellen Clancy, Michal Zielinski, Martin Steinegger, Michalina Pacholska, Tamas Berghammer, Sebastian Bodenstein, David Silver, Oriol Vinyals, Andrew~W. Senior, Koray Kavukcuoglu, Pushmeet Kohli, and Demis Hassabis.
\newblock Highly accurate protein structure prediction with alphafold.
\newblock \emph{Nature}, 2021.

\bibitem[Kazemnejad et~al.(2024)Kazemnejad, Padhi, Natesan~Ramamurthy, Das, and Reddy]{kazemnejad2024impact}
Amirhossein Kazemnejad, Inkit Padhi, Karthikeyan Natesan~Ramamurthy, Payel Das, and Siva Reddy.
\newblock The impact of positional encoding on length generalization in transformers.
\newblock \emph{NeurIPS}, 2024.

\bibitem[Li et~al.(2024)Li, You, Guruganesh, Ainslie, Ontanon, Zaheer, Sanghai, Yang, Kumar, and Bhojanapalli]{li2024functional}
Shanda Li, Chong You, Guru Guruganesh, Joshua Ainslie, Santiago Ontanon, Manzil Zaheer, Sumit Sanghai, Yiming Yang, Sanjiv Kumar, and Srinadh Bhojanapalli.
\newblock Functional interpolation for relative positions improves long context transformers.
\newblock In \emph{ICLR}, 2024.

\bibitem[Martins \& Astudillo(2016)Martins and Astudillo]{martins2016softmax}
Andre Martins and Ramon Astudillo.
\newblock From softmax to sparsemax: A sparse model of attention and multi-label classification.
\newblock In \emph{ICML}, 2016.

\bibitem[Peebles \& Xie(2023)Peebles and Xie]{peebles2023scalable}
William Peebles and Saining Xie.
\newblock Scalable diffusion models with transformers.
\newblock In \emph{ICCV}, 2023.

\bibitem[Peng et~al.(2024)Peng, Quesnelle, Fan, and Shippole]{peng2024yarn}
Bowen Peng, Jeffrey Quesnelle, Honglu Fan, and Enrico Shippole.
\newblock Yarn: Efficient context window extension of large language models.
\newblock In \emph{ICLR}, 2024.

\bibitem[Press et~al.(2022)Press, Smith, and Lewis]{press2022train}
Ofir Press, Noah~A. Smith, and Mike Lewis.
\newblock Train short, test long: Attention with linear biases enables input length extrapolation.
\newblock In \emph{ICLR}, 2022.

\bibitem[Radford et~al.(2023)Radford, Kim, Xu, Brockman, McLeavey, and Sutskever]{radford2023robust}
Alec Radford, Jong~Wook Kim, Tao Xu, Greg Brockman, Christine McLeavey, and Ilya Sutskever.
\newblock Robust speech recognition via large-scale weak supervision.
\newblock In \emph{ICML}, 2023.

\bibitem[Ruoss et~al.(2023)Ruoss, Del{\'e}tang, Genewein, Grau-Moya, Csord{\'a}s, Bennani, Legg, and Veness]{ruoss2023randomized}
Anian Ruoss, Gr{\'e}goire Del{\'e}tang, Tim Genewein, Jordi Grau-Moya, R{\'o}bert Csord{\'a}s, Mehdi Bennani, Shane Legg, and Joel Veness.
\newblock Randomized positional encodings boost length generalization of transformers.
\newblock In \emph{ACL}, 2023.

\bibitem[Shaw et~al.(2018)Shaw, Uszkoreit, and Vaswani]{shaw2018self}
Peter Shaw, Jakob Uszkoreit, and Ashish Vaswani.
\newblock Self-attention with relative position representations.
\newblock \emph{arXiv preprint arXiv:1803.02155}, 2018.

\bibitem[Su et~al.(2024)Su, Ahmed, Lu, Pan, Bo, and Liu]{su2024roformer}
Jianlin Su, Murtadha Ahmed, Yu~Lu, Shengfeng Pan, Wen Bo, and Yunfeng Liu.
\newblock Roformer: Enhanced transformer with rotary position embedding.
\newblock \emph{Neurocomputing}, 2024.

\bibitem[Tan et~al.(2024)Tan, Shen, Yang, Courville, and Panda]{tan2024stick}
Shawn Tan, Yikang Shen, Songlin Yang, Aaron Courville, and Rameswar Panda.
\newblock Stick-breaking attention.
\newblock \emph{arXiv preprint arXiv:2410.17980}, 2024.

\bibitem[Vaswani et~al.(2017)Vaswani, Shazeer, Parmar, Uszkoreit, Jones, Gomez, Kaiser, and Polosukhin]{vaswani2017attention}
Ashish Vaswani, Noam Shazeer, Niki Parmar, Jakob Uszkoreit, Llion Jones, Aidan~N Gomez, {\L}ukasz Kaiser, and Illia Polosukhin.
\newblock Attention is all you need.
\newblock In \emph{NeurIPS}, 2017.

\bibitem[Veli{\v{c}}kovi{\'c} et~al.(2022)Veli{\v{c}}kovi{\'c}, Badia, Budden, Pascanu, Banino, Dashevskiy, Hadsell, and Blundell]{velivckovic2022clrs}
Petar Veli{\v{c}}kovi{\'c}, Adri{\`a}~Puigdom{\`e}nech Badia, David Budden, Razvan Pascanu, Andrea Banino, Misha Dashevskiy, Raia Hadsell, and Charles Blundell.
\newblock The clrs algorithmic reasoning benchmark.
\newblock In \emph{ICML}, 2022.

\bibitem[Veli{\v{c}}kovi{\'c} et~al.(2024)Veli{\v{c}}kovi{\'c}, Perivolaropoulos, Barbero, and Pascanu]{velivckovic2024softmax}
Petar Veli{\v{c}}kovi{\'c}, Christos Perivolaropoulos, Federico Barbero, and Razvan Pascanu.
\newblock softmax is not enough (for sharp out-of-distribution).
\newblock \emph{arXiv preprint arXiv:2410.01104}, 2024.

\bibitem[Xu et~al.(2015)Xu, Ba, Kiros, Cho, Courville, Salakhudinov, Zemel, and Bengio]{xu2015show}
Kelvin Xu, Jimmy Ba, Ryan Kiros, Kyunghyun Cho, Aaron Courville, Ruslan Salakhudinov, Rich Zemel, and Yoshua Bengio.
\newblock Show, attend and tell: Neural image caption generation with visual attention.
\newblock In \emph{ICML}, 2015.

\bibitem[Ye et~al.(2024)Ye, Dong, Xia, Sun, Zhu, Huang, and Wei]{ye2024differential}
Tianzhu Ye, Li~Dong, Yuqing Xia, Yutao Sun, Yi~Zhu, Gao Huang, and Furu Wei.
\newblock Differential transformer.
\newblock \emph{arXiv preprint arXiv:2410.05258}, 2024.

\bibitem[Zhou et~al.(2023)Zhou, Bradley, Littwin, Razin, Saremi, Susskind, Bengio, and Nakkiran]{zhou2023algorithms}
Hattie Zhou, Arwen Bradley, Etai Littwin, Noam Razin, Omid Saremi, Josh Susskind, Samy Bengio, and Preetum Nakkiran.
\newblock What algorithms can transformers learn? a study in length generalization.
\newblock \emph{arXiv preprint arXiv:2310.16028}, 2023.

\bibitem[Zhou et~al.(2022)Zhou, Liu, Qiao, Xiang, and Loy]{zhou2022domain}
Kaiyang Zhou, Ziwei Liu, Yu~Qiao, Tao Xiang, and Chen~Change Loy.
\newblock Domain generalization: A survey.
\newblock \emph{PAMI}, 2022.

\bibitem[Zhou et~al.(2024)Zhou, Alon, Chen, Wang, Agarwal, and Zhou]{zhou2024transformers}
Yongchao Zhou, Uri Alon, Xinyun Chen, Xuezhi Wang, Rishabh Agarwal, and Denny Zhou.
\newblock Transformers can achieve length generalization but not robustly.
\newblock \emph{arXiv preprint arXiv:2402.09371}, 2024.

\end{thebibliography}
\bibliographystyle{iclr2025_conference}

\newpage

\appendix
\setcounter{section}{0} 
\setcounter{proposition}{0}
\renewcommand{\thesection}{\Alph{section}}

\section{Proof of~\cref{proposition}}
\label{sec:app_proof}
\begin{proposition}[The vanishing variance problem]
Consider a \textbf{trained} attention module with weights $\mathbf{W}_Q, \mathbf{W}_K, \mathbf{W}_V, \mathbf{W}_O$. Let $\mathbf{X} = \left[\mathbf{x}_1 \Vert \mathbf{x}_2 \Vert \dots \Vert \mathbf{x}_N \right ]^{\top}$ denote an input sequence of length $N$.
If (1) $\mathbf{x}_1, \mathbf{x}_2, \dots, \mathbf{x}_N \overset{\text{i.i.d}}{\sim} \mathcal{X}$, a distribution over a \textbf{finite} vocabulary, and (2) 
$\mathbb{E}_{\mathbf{x} \sim \mathcal{X}}[\mathbf{W}_V \mathbf{x}] = \mathbf{0}$, then for a \textbf{fixed} query $\mathbf{y}$ and a \textbf{fixed} feature $d$,
\begin{equation*}
    \begin{aligned}
\lim_{N \to \infty}
\operatorname{Var}_{ (\mathbf{x}_{1}, \mathbf{x}_{2}, \ldots, \mathbf{x}_{N})  \sim \mathcal{X}^N}
\left(\left[\operatorname{softmax}\left (\frac{\mathbf{Q}\mathbf{K}^\top}{\sqrt{D}} \right)\mathbf{V}\right]_{d} \right) = 0,
    \end{aligned}
\end{equation*}
where $\mathbf{x}_n, \mathbf{y} \in \mathbb{R}^D$ and $\mathbf{Q} \in \mathbb{R}^{1\times D}, \mathbf{K} \in \mathbb{R}^{N\times D}, \mathbf{V} \in \mathbb{R}^{N\times D}$ are intermediate results in $\operatorname{Attn}(\mathbf{X}, [\mathbf{y}])$.

Informally, for a fixed component $d$ in the attention outputs, its variance over input sequences of length $N$, where each sequence consists of $N$ independently and identically distributed (i.i.d.) tokens, vanishes as $N \to \infty$.
\end{proposition}

\begin{proof}
Let $\mathbf{V}_{n,d} = [\mathbf{W}_V \mathbf{x}_n]_d$. Firstly, we argue that $\mathbf{V}_{k,d}$ and $\mathbf{V}_{l,d}$ are independent for $k \neq l$. Let $\pi_d : \mathbb{R}^D \to \mathbb{R}$ be the projection onto the $d$-th coordinate, namely $\pi_d (\mathbf{v}) = \mathbf{v}_d $ for $\mathbf{v} \in \mathbb{R}^{D}$. Observe that $\mathbf{V}_{k,d} = \pi_{d} (\mathbf{W}_V \mathbf{x}_k) = (\pi_{d} \circ  \mathbf{W}_V) (\mathbf{x}_k) $ and $\mathbf{V}_{l,d} = \pi_{d} (\mathbf{W}_V \mathbf{x}_l) = (\pi_{d} \circ  \mathbf{W}_V) (\mathbf{x}_l) $. By assumption (1), $\mathbf{x}_k$ and $\mathbf{x}_l$ are independent for $k \neq l$. Since $\mathbf{V}_{k,d}$, $\mathbf{V}_{l,d} $ are measurable functions of independent random variables, they are independent for $k \neq l$. By assumption, $\mathbf{x}_k$ and $\mathbf{x}_l$ are identically distributed   for every $k, l$, so $\mathbf{V}_{k, d}$ and  $\mathbf{V}_{l, d}$ are identically distributed. Thus, $ \operatorname{Var} [\mathbf{V}_{k, d} ]$ depends only on $d$, not on $k$. Set $\sigma_d^2 =  \operatorname{Var} [\mathbf{V}_{k, d} ]$. $\sigma_d^2$ is finite since the vocabulary is finite and thus compact and bounded.

Let $\mathbf{A}$ denote the attention weights $\operatorname{softmax}\left (\frac{\mathbf{Q}\mathbf{K}^\top}{\sqrt{D}} \right) \in \mathbb{R}^{1\times N}$ as the query sequence $[\mathbf{y}]$ consists of only a \emph{single} item. Let $\mathbf{A}_n$ denote the $n$-th element of $\mathbf{A}$. We have
\begin{equation*}
    \begin{aligned}
        \operatorname{Var} \left ( \sum_{n=1}^N \mathbf{A}_{n} \mathbf{V}_{n, d} \right) &= \mathbb{E} \left [ \left ( \sum_{n=1}^N \mathbf{A}_{n} \mathbf{V}_{n, d} \right)^2  \right] - \left ( \mathbb{E} \left [\sum_{n=1}^N \mathbf{A}_{n} \mathbf{V}_{n, d}  \right] \right)^2  \\
        &\leq \mathbb{E} \left [ \left ( \sum_{n=1}^N \mathbf{A}_{n} \mathbf{V}_{n, d} \right)^2  \right] \\
        &= \mathbb{E} \left [  \sum_{n=1}^N \mathbf{A}_{n}^2 \mathbf{V}_{n, d}^2  \right] +  \mathbb{E} \left [  \sum_{ 1\leq k, l \leq N, k \neq l}  \mathbf{A}_{k} \mathbf{A}_{l} \mathbf{V}_{k, d} \mathbf{V}_{l, d} \right] \\
        &\leq \mathbb{E}  \left [ \sum_{n=1}^N (\operatorname{max}_{1 \leq n \leq N}  \mathbf{A}_{n})^2  \mathbf{V}_{n, d}^2 \right] + \mathbb{E} \left [  \sum_{ 1\leq k, l \leq N, k \neq l, \mathbf{V}_{k, d} \mathbf{V}_{l, d} \geq 0  }  \mathbf{A}_{k} \mathbf{A}_{l} \mathbf{V}_{k, d} \mathbf{V}_{l, d} \right] \\
        &\leq (\operatorname{max}_{1 \leq n \leq N}  \mathbf{A}_{n})^2   \sum_{n=1}^N \mathbb{E} \left [  \mathbf{V}_{n, d}^2 \right ] + \mathbb{E} \left [   \sum_{ 1\leq k, l \leq N, k \neq l, \mathbf{V}_{k, d} \mathbf{V}_{l, d} \geq 0}  \mathbf{V}_{k, d} \mathbf{V}_{l, d}  \right ] \\
        &\leq N (\operatorname{max}_{1 \leq n \leq N}  \mathbf{A}_{n})^2  \sigma_d^2  +  \sum_{ 1\leq k, l \leq N, k \neq l, \mathbf{V}_{k, d} \mathbf{V}_{l, d} \geq 0}  \mathbb{E} \left [  \mathbf{V}_{k, d} \mathbf{V}_{l, d}  \right ] \\ 
        &\leq N  (\operatorname{max}_{1 \leq n\leq N}  \mathbf{A}_{n})^2  \sigma_d^2 +  \sum_{ 1\leq k, l \leq N, k \neq l, \mathbf{V}_{k, d}, \mathbf{V}_{l, d} \geq 0}  \mathbb{E} \left [  \mathbf{V}_{k, d}   \right ]   \mathbb{E} \left [  \mathbf{V}_{l, d}   \right ] \\ 
        &\leq N  (\operatorname{max}_{1 \leq n \leq N}  \mathbf{A}_{n})^2  \sigma_d^2 
    \end{aligned}
\end{equation*}

In the derivation above, we used the fact that $\mathbf{A}_{n} \in [0,1]$ and that for $k \neq l$, $\mathbf{V}_{k,d}$ and $\mathbf{V}_{l,d} $ are independent. We also used the assumption that for every $1 \leq k \leq N$, $ \mathbb{E} \left [  \mathbf{V}_{k, d}   \right ] = 0$. Since the tokens come from a finite dictionary, and since $\mathbf{x} \to  \mathbf{W}_Q \mathbf{x}$ and  $\mathbf{x} \to  \mathbf{W}_K \mathbf{x}$ are continuous functions on compact domain (dictionary is finite), the logits $\langle \mathbf{W}_Q \mathbf{x}_i, \mathbf{W}_K \mathbf{x}_j \rangle$ are bounded, because they are continuous image of a compact set and every compact set on the real line is closed and bounded.
By Lemma 2.1 of \citet{velivckovic2024softmax}, there exist a constant $C > 0$ and $N_{0} \in \mathbb{N}$, such that for every $N \geq N_{0}$, $ (\operatorname{max}_{1 \leq n \leq N}  \mathbf{A}_{n})^2 < \frac{C}{N^2}$. Then for every  $N \geq N_{0}$, 

\begin{equation*}
    \begin{aligned}
        \operatorname{Var} \left ( \sum_{n=1}^N \mathbf{A}_{n} \mathbf{V}_{n, d} \right) 
        &\leq N \sigma^2 \frac{C}{N^2} =\sigma^2 \frac{C}{N}.
    \end{aligned}
\end{equation*}


Let $\epsilon > 0$. There exists $N_1 \in \mathbb{N}$ such that for every $N \geq N_1$, $\sigma^2 \frac{C}{N} < \epsilon$. Then for every $N \geq \operatorname{max} (N_0, N_1)$, 
\begin{equation*}
    \begin{aligned}
        \operatorname{Var} \left ( \sum_{n=1}^N \mathbf{A}_{n} \mathbf{V}_{n, d} \right)
        &< \epsilon.
    \end{aligned}
\end{equation*}
\end{proof}

\setlength{\tabcolsep}{1pt}

\begin{table}[]
    \caption{
    \textbf{Results (\%) on the $\operatorname{argmax}$ retrieval task.} 
    Results are averaged over $100$ runs with different random seeds. 
    $p$-values are computed using a paired $t$-test.
    Entries highlighted in \colorbox{green!15}{green} indicate those with in-distribution sequence lengths.
    }
    \tablestyle{2.4pt}{1.0}
    \centering
    \begin{tabular}{lccccccccccc}
        \toprule
        {\bf Model} & \cellcolor{green!15} $2^4$ & \cellcolor{green!15} $2^5$ & \cellcolor{green!15} $2^6$ & \cellcolor{green!15} $2^7$ & \cellcolor{green!15} $2^8$ & $2^9$ & $2^{10}$ & $2^{11}$ & $2^{12}$ & $2^{13}$ & $2^{14}$\\
        \midrule
        Baseline & \cellcolor{green!15} {$99.8$} & \cellcolor{green!15} $99.8$ & \cellcolor{green!15} ${99.6}$ & \cellcolor{green!15} $99.3$ & \cellcolor{green!15} $98.5$ & $97.1$ & $94.4$ & $89.1$ & $79.9$ & $65.8$ & $47.8$ \\
        Baseline (+ LN) & \cellcolor{green!15} $\mathbf{99.9}$ & \cellcolor{green!15} $\mathbf{99.9}$ & \cellcolor{green!15} $\mathbf{99.8}$ & \cellcolor{green!15} $\mathbf{99.5}$ & \cellcolor{green!15} $\mathbf{99.1}$ & $\mathbf{98.1}$ & $\mathbf{96.2}$ & $\mathbf{92.8}$ & $\mathbf{86.3}$ & $\mathbf{75.1}$ & $\mathbf{58.7}$ \\
        \midrule
        \textcolor{gray}{$p$-value} & \cellcolor{green!15} $\textcolor{gray}{1/10^{8}}$ & \cellcolor{green!15} $\textcolor{gray}{6/10^6}$ & \cellcolor{green!15} $\textcolor{gray}{1/10^{4}}$ & \cellcolor{green!15} $\textcolor{gray}{2/10^{4}}$ & \cellcolor{green!15} $\textcolor{gray}{9/10^{8}}$ & $\textcolor{gray}{5/10^{8}}$ & $\textcolor{gray}{3/10^8}$ & $\textcolor{gray}{5/10^{10}}$ & $\textcolor{gray}{1/10^{11}}$ & $\textcolor{gray}{3/10^{12}}$ & $\textcolor{gray}{1/10^{13}}$ \\
        \bottomrule 
    \end{tabular}
    \label{tab:comparison-256-argmax}
\end{table}

\setlength{\tabcolsep}{1pt}

\begin{table}[]
    \caption{
    \textbf{Results (\%) on the dictionary lookup task}.
    Results are averaged over $100$ runs with different random seeds. 
    $p$-values are computed using a paired $t$-test. 
    Entries highlighted in \colorbox{green!15}{green} indicate those with in-distribution sequence lengths.
    }
    \tablestyle{2.4pt}{1.0}
    \centering
    \begin{tabular}{lccccccccccc}
        \toprule
        {\bf Model} & \cellcolor{green!15} $2^4$ & \cellcolor{green!15} $2^5$ & \cellcolor{green!15} $2^6$ & \cellcolor{green!15} $2^7$ & \cellcolor{green!15} $2^8$ & $2^9$ & $2^{10}$ & $2^{11}$ & $2^{12}$ & $2^{13}$ & $2^{14}$\\
        \midrule
        Baseline & \cellcolor{green!15} {$\mathbf{99.9}$} & \cellcolor{green!15} $\mathbf{99.9}$ & \cellcolor{green!15} ${99.8}$ & \cellcolor{green!15} $99.7$ & \cellcolor{green!15} $99.5$ & $99.1$ & $98.3$ & $96.5$ & $93.5$ & $87.7$ & $77.8$\\
        Baseline (+ LN) & \cellcolor{green!15} $\mathbf{99.9}$ & \cellcolor{green!15} $\mathbf{99.9}$ & \cellcolor{green!15} $\mathbf{99.9}$ & \cellcolor{green!15} $\mathbf{99.8}$ & \cellcolor{green!15} $\mathbf{99.6}$ & $\mathbf{99.3}$ & $\mathbf{98.7}$ & $\mathbf{97.5}$ & $\mathbf{95.0}$ & $\mathbf{90.4}$ & $\mathbf{82.6}$ \\
        \midrule
        \textcolor{gray}{$p$-value} & \cellcolor{green!15} $\textcolor{gray}{2/10^{1}}$ & \cellcolor{green!15} $\textcolor{gray}{1/10^2}$ & \cellcolor{green!15} $\textcolor{gray}{5/10^2}$ & \cellcolor{green!15} $\textcolor{gray}{2/10^3}$ & \cellcolor{green!15} $\textcolor{gray}{3/10^3}$ & $\textcolor{gray}{6/10^{5}}$ & $\textcolor{gray}{4/10^{7}}$ & $\textcolor{gray}{8/10^9}$ & $\textcolor{gray}{2/10^{10}}$ & $\textcolor{gray}{2/10^{10}}$ & $\textcolor{gray}{2/10^{13}}$ \\
        \bottomrule 
    \end{tabular}
    \label{tab:comparison-256-dict-lookup}
\end{table}
\setlength{\tabcolsep}{1pt}

\begin{table}[]
    \caption{
    \textbf{Ablations} on different normalization strategies on the dictionary lookup task.
    Relative results (\%) compared to the \emph{Baseline} ($\triangle$) are reported.
    }
    \tablestyle{3.0pt}{1.0}
    \centering
    \begin{tabular}{lccccccccccc}
        \toprule
        {\bf Model} & \cellcolor{green!15} $2^4$ & $2^5$ & $2^6$ & $2^7$ & $2^8$ & $2^9$ & $2^{10}$ & $2^{11}$ & $2^{12}$ & $2^{13}$ & $2^{14}$\\
        \midrule
        $\triangle$ (+ std.) & \cellcolor{green!15} $\mathbf{+0.09}$ & $+0.14$ & $+0.23$ & $+0.54$ & $+1.08$ & $+2.27$ & $+3.49$ & $+4.78$ & $+5.14$ & $+5.65$ & $\mathbf{+4.57}$ \\
        $\triangle$ (+ LN) & \cellcolor{green!15} $\mathbf{+0.09}$ & $\mathbf{+0.20}$ & $\mathbf{+0.30}$ & $\mathbf{+0.64}$ & $\mathbf{+1.22}$ & $\mathbf{+2.55}$ & $\mathbf{+4.06}$ & $\mathbf{+4.86}$ & $\mathbf{+5.38}$ & $\mathbf{+5.72}$ & $+4.51$ \\    
        \midrule
        \textcolor{gray}{$p$-value} & \cellcolor{green!15} $\textcolor{gray}{1.0}$ & $\textcolor{gray}{0.3}$ & $\textcolor{gray}{0.6}$ & $\textcolor{gray}{0.6}$ & $\textcolor{gray}{0.6}$ & $\textcolor{gray}{0.5}$ & $\textcolor{gray}{0.3}$ & $\textcolor{gray}{0.9}$ & $\textcolor{gray}{0.6}$ & $\textcolor{gray}{0.9}$ & $\textcolor{gray}{0.9}$\\
        \bottomrule 
    \end{tabular}
    \label{tab:ablation-dict-lookup}
\end{table}
\section{More Experimental Details}
\label{sec:supp_details}

\subsection{Implementation Details}
\paragraph{$\operatorname{argmax}$ retrieval.}
We follow~\citet{velivckovic2024softmax} and train the same neural network architecture in PyTorch for $100,000$ gradient steps with the same hyper-parameter setup. We also follow~\citet{velivckovic2024softmax} to generate data of varying number of items to train and test the model.

\paragraph{$\operatorname{dictionary}$  lookup.} The network architecture is the same as the $\arg\max$ retrieval task. We generate data for training and evaluation in the following way: for each item of the length-$N$ sequence, we sample a value class $c_V \sim \mathcal{U}\{1, \dots, C_V\}$ i.i.d at random; each item also has a key class $1 \leq c_K \leq C_K$. The key classes of all $N$ items in the sequence are sampled \emph{without} replacement.
In our experiments, $C_K = 16384$ and $C_V = 64$.

The features of each item $\mathbf{x}_i$ is defined as $\operatorname{Emb}_K(c_K) \parallel \operatorname{Emb}_V(c_V)$, \ie, the concatenation of the embeddings of the key class and the value class. The embedding vectors for each (key and value) class are optimized jointly with the attention network.

The query sequence in our case is guaranteed to be of length $1$. We sample a key class present in the input sequence and use its embedding vector as the query.

For this task, we found that the optimization usually converges within $10,000$ gradient steps. We train the attention network, together with the embedding vectors, in PyTorch for $10,000$ steps with the same hyper-parameter setup as the $\arg\max$ task.

\subsection{Results When Training on More Diverse Sequence Lengths}
To validate the utility of normalization when the length gap between the training sequences and the test sequences is smaller, we follow the same experimental setup as in~\cref{sec:experiments}, but sample sequences of up to $256$ items during training.
We found it beneficial to gradually increase the length of the sequences sampled throughout training, as is commonly done during pre-training of frontier LLMs~\citep{llama3modelcard}.
The results are reported in~\cref{tab:comparison-256-argmax} and~\cref{tab:comparison-256-dict-lookup}.
With layer normalization, the accuracies on out-of-distribution sequence lengths are significantly higher than without on both tasks, demonstrating the importance of normalization for length generalization over various training settings.
\subsection{Ablations on the Dictionary Lookup Task}
\label{sec:additional-ablations}
Ablation results on the dictionary lookup task are shown in~\cref{tab:ablation-dict-lookup}, which are consistent with the results on the $\arg\max$ retrieval task presented in~\cref{sec:ablations}.
However, on this task, the performance of standardization and layer normalization is more similar, as indicated by the larger $p$-values, suggesting weaker statistical evidence for a significant difference.

\end{document}